\documentclass{IEEEtran}
\usepackage{cite}
\usepackage{amsmath,amssymb,amsfonts}
\usepackage{algorithmic}
\usepackage{graphicx}
\usepackage{amsthm}
\usepackage{breqn}
\usepackage{textcomp}
\usepackage{color}
\usepackage{xcolor}
\usepackage{hyperref}
\usepackage{xcolor}
\usepackage{algorithm}
\usepackage{xcolor}
\usepackage{booktabs}
\usepackage{breqn}
\usepackage{subcaption}
\usepackage{algorithm}
\usepackage{amsmath}
\usepackage{pdflscape}
\usepackage{hyperref}

\newtheorem{theorem}{\bf Theorem}[section]

\newtheorem{remark}[theorem]{\bf Remark}
\newtheorem{lemma}[theorem]{\bf Lemma}
\usepackage{algorithm}
\usepackage{multirow}
\usepackage{graphicx}
\usepackage{float}
\usepackage{placeins}

\newcommand{\beano}{\begin{eqnarray*}}
	\newcommand{\eeano}{\end{eqnarray*}}

\makeatletter

\def\BibTeX{{\rm B\kern-.05em{\sc i\kern-.025em b}\kern-.08em
    T\kern-.1667em\lower.7ex\hbox{E}\kern-.125emX}}
\begin{document}
\title{\textsc{Two-Step Q-learning
}}
\author{V. Antony Vijesh\IEEEauthorrefmark {1}, Shreyas S R\IEEEauthorrefmark {2}
}

\maketitle

\begin{abstract}Q-learning is a stochastic approximation version of the classic value iteration. The literature has established that Q-learning suffers from both maximization bias and slower convergence. Recently, multi-step algorithms have shown practical advantages over existing methods. This paper proposes a novel off-policy two-step Q-learning algorithms, without importance sampling. With suitable assumption it was shown that, iterates in the proposed two-step Q-learning is bounded and converges almost surely to the optimal Q-values. This study also address the convergence analysis of the smooth version of two-step Q-learning, i.e., by replacing max function with the log-sum-exp function. The proposed algorithms are robust and easy to implement. Finally, we test the proposed algorithms on benchmark problems such as the roulette problem, maximization bias problem, and randomly generated Markov decision processes and compare it with the existing methods available in literature. Numerical experiments demonstrate the superior performance of both the two-step Q-learning and its smooth variants.
\end{abstract}

\footnotetext[1]{
	Department  of Mathematics, IIT Indore, Simrol, Indore-452020, Madhya Pradesh, India. \texttt{Email:vijesh@iiti.ac.in}}
\footnotetext[2]{Research Scholar, Department of Mathematics, IIT Indore, Shreyas S R is sincerely grateful to CSIR-HRDG, India (Grant No. 09/1022(0088)2019-EMR-I) for the financial support. \texttt{Email:shreyassr123@gmail.com, phd1901241006@iiti.ac.in.}}

\section{Introduction}
Making an optimal decision in a stochastic environment is an interesting and a challenging problem. Unlike many machine learning algorithms, Reinforcement Learning (RL) algorithms interact with an environment to learn an optimal decision or policy \cite{MR3889951}. In recent days various RL algorithms have been developed and they find their applications in different areas of science and technology \cite{fawzi2022discovering}\cite{silver2016mastering}. Q-learning is one of the classic off-policy single-step RL algorithms proposed by Watkins to find an optimal policy in an uncertain environment \cite{watkins1989learning}. Even though the Q-learning algorithm is effectively implemented in many real-world applications, such as robotics \cite{yang2004mobile}, traffic signal control \cite{5658157}, and agent based production scheduling \cite{wang2005application}, it suffers from maximization bias \cite{hasselt2010double}\cite{8695133}, and slower convergence \cite{azar2011speedy}\cite{kamanchi2019successive}. In particular, it has been observed that in certain situations, such as the roulette example mentioned in \cite{8695133}, Q-learning overestimates the optimal value even after large number of iterations. Consistent efforts are being made to address these issues effectively (for example, \cite{azar2011speedy}\cite{hasselt2010double}\cite{kamanchi2019successive}\cite{8695133}).
\\
Developing multi-step temporal difference algorithms have always been an interesting research topic in RL \cite{watkins1989learning, sutton1988learning,sutton1998introduction,peng1994incremental}. Initially, multi-step temporal difference algorithm for prediction problem is developed in \cite{sutton1988learning}. Also, in his seminal work, Watkins \cite{watkins1989learning}, explores the notion of multi-step returns and the trade-off between bias and variance. In \cite{watkins1989learning}, Watkin's also proposed a multi-step algorithm, a generalisation of the well known Q-learning algorithm for control problem without rigorous convergence analysis. This algorithm is also called as Watkin's Q($\lambda$). The drawback of the Watkin's Q($\lambda$) algorithm is that, it truncate the trajectory as soon as an exploratory action is chosen instead of a greedy action. Consequently, the Watkin's Q($\lambda$) will lack much of its advantage, if the exploratory actions are frequent in the initial learning phase \cite{sutton1998introduction}. By relaxing the truncation procedure in the Watkin's Q($\lambda$) algorithm, Peng and Williams proposed a multi-step Q-learning algorithm in \cite{peng1994incremental}. As mentioned in \cite{sutton1998introduction}, this algorithm can be viewed as an hybrid of SARSA($\lambda$) and Watkin's Q($\lambda$). In a similar vein, a multi-step Q-learning algorithm, naive Q($\lambda$) has been proposed in \cite{sutton1998introduction}. Naive Q($\lambda$) and Peng's and William's Q($\lambda$) do not terminate the backup when an exploratory action is taken but differ in evaluating temporal difference error. Even though the above multi-step RL algorithms for control were proposed long ago, the theoretical understanding of the multi-step algorithms regarding convergence was an open question until recently. In 2016, Harutyunyan et al. \cite{harutyunyan2016q}, constructed a novel off-policy operator, and the convergence of naive Q($\lambda$) was established under some restrictions on the behavioural policy, and the parameter $\lambda$ \cite{harutyunyan2016q}. In the same year, Munos et al. \cite{munos2016safe}, developed a multi-step algorithm Retrace($\lambda$) using a more general operator. The operator in \cite{munos2016safe}, unifies several existing multi-step algorithms. Further, sufficient conditions were provided to ensure the convergence of this multi-step method for control problems. As a corollary, the convergence of Watkin's Q($\lambda$) method is obtained in \cite{munos2016safe}. Recently, the theoretical convergence analysis of Peng's and William's Q($\lambda$) is discussed  in \cite{kozuno2021revisiting}. In \cite{kozuno2021revisiting}, Kozuno et al. showed that the conjecture by Sutton and Barto \cite{sutton1998introduction} regarding the limit of Peng's and William's Q($\lambda$) method need not be true. More specifically, it was shown that if the behaviour policy is fixed, then the Q-values produced by the Peng's and William's  Q($\lambda$) need not be the optimal value. Note that unlike Q-learning, the above multi-step algorithms require suitable additional assumption on behavioural policy for their convergence \cite{harutyunyan2016q,kozuno2021revisiting,munos2016safe}. Earlier works have demonstrated the practical usage of multi-step algorithms in \cite{cichosz1994truncating}, and \cite{van2016true}. In addition, recent literature observes that, multi-step algorithms combined with deep learning are more suitable for several control tasks compared to single-step algorithms \cite{mousavi2017applying,harb2017investigating,hessel2018rainbow,barth2018distributed,kapturowski2018recurrent,daley2019reconciling}. Recently, clubbing the algorithms multi-step SARSA and tree backup algorithms, an interesting multi-step algorithm Q($\sigma$) was studied in \cite{de2018multi}. The convergence analysis was also discussed for the one-step Q($\sigma$) algorithm in \cite{de2018multi}. Also, in \cite{6796861}, the convergence of TD($\lambda$) algorithms for prediction problems was presented using the techniques from stochastic approximation. It was mentioned in \cite{de2018multi} that, the multi-step temporal difference algorithms in \cite{de2018multi} and \cite{6796861} could decrease the bias of the update at the cost of an increase in the variance.     \\
Note that, the on-policy algorithms are a well-known approach to address the problem of policy evaluation \cite{MR3889951}. However, due to the advantages of off-policy algorithms, efforts are being made to develop the off-policy multi-step algorithms for both policy evaluation and control problems \cite{ harutyunyan2016q,munos2016safe,precup2000eligibility, mahmood2017multi}. Off-policy algorithms refer to, learning about the target policy, while data gets generated according to a behavioural policy. Off-policy algorithms may not require any relationship between the behaviour and target policies, leading to higher flexibility in exploration strategies. Most of the off-policy algorithms use an importance sampling ratio. However, it is observed that the off-policy algorithms using importance sampling suffer from high variance \cite{MR3889951}. Consequently, the development of off-policy algorithms, which do not use importance sampling becomes an important research topic. To this end, this manuscript proposes convergent off-policy multi-step algorithms that do not use importance sampling. The proposed two-step Q-learning algorithm is a hybrid approach that combines the elements from Q-learning and a two-step tree backup algorithm. More specifically, the proposed two-step Q-learning algorithm uses additional samples, like a two-step tree backup algorithm, while retaining the terms from Q-learning. Interestingly, the additional term incorporated in the two-step Q-learning makes a considerable difference in the learning process. Moreover, the importance of the additional term in the algorithm becomes inconsequential as the learning progresses to avoid the issue of variance. Under suitable assumptions, the boundedness and convergence of the proposed algorithms are provided. 
More specifically, the almost sure convergence of the proposed iterative scheme is obtained using the stochastic approximation result from \cite{singh2000convergence}.  \\
Studying Q-learning algorithm by substituting the maximum operator with a suitable smooth approximate operator attracts considerable attention in the recent literature (for example, see \cite{asadi2017alternative,haarnoja2017reinforcement,9540362,nachum2017bridging,song2019revisiting}). Owing to the advantages mentioned in the literature, apart from the regular two-step Q-learning (TSQL) algorithm, we propose a smooth two-step algorithm (S-TSQL) and discuss its convergence. Finally, both TSQL and S-TSQL are tested using a classic maximization bias example, roulette example, and a hundred randomly generated Markov decision processes (MDPs). The empirical performance demonstrates the proposed algorithms' superior performance compared to the existing algorithms. 
\subsection{Literature Review}
Several alternative algorithms are proposed to improve the Q-learning algorithm to reduce bias and fasten the convergence. In \cite{hasselt2010double}, the double Q-learning algorithm is proposed, which uses two Q-estimators to estimate the maximum expected action value, and it is shown to underestimate maximum expected action values. By adding a suitable bias corrected term at each step of the Q-learning iteration, a novel bias-corrected Q-learning algorithm was studied in \cite{8695133}. In this approach \cite{8695133},  the bias due to the maximum operator is handled by controlling the bias due to the stochasticity in the reward and the transition. Self-correcting Q-learning \cite{zhu2021self} is a Q-learning variant designed to mitigate maximization bias by utilizing self-correcting estimator. It is worth to mention that the self correcting estimator at $(n+1)^{th}$ step depend on the Q-values at $n$ and $n-1$ steps. It is interesting to note that the influence of the bias corrected term in \cite{8695133} as well as the self correcting estimator in \cite{zhu2021self} is noticeable only in the early stages, and as the iterates progress, both of their influence in their respective algorithms is inconsequential.    In \cite{zhang2017weighted}, authors propose weighted double Q-learning to balance the overestimation by single estimator and underestimation by double estimator.  Speedy Q-learning (SQL) was introduced to tackle the issue of slower convergence in Q-learning \cite{azar2011speedy}. In each iteration of SQL, two successive estimates of the Q-function are used to update the iteration. Further, the boundedness and almost sure convergence of SQL is discussed in \cite{azar2011speedy}. Recently, by clubbing SQL and successive over-relaxation method, an interesting version of Q-learning is discussed in \cite{12345} for a specific class of MDPs.   \\It is worth mentioning that the order of convergence of Q-value iteration is accelerated by replacing the maximum operator with the log-sum-exp (LSE) operator \cite{9540362}\cite{MR1315986}. In this direction, the performance of RL algorithms using an alternate smooth maximum operator called as a mellowmax operator for model-free algorithms has been discussed in \cite{asadi2017alternative}. In addition, \cite{song2019revisiting} investigated the Bellman operator with various smooth functions combined with deep Q-network and double deep Q-network, and their performance was evaluated using Atari games.  The boundedness of the Q-values plays a crucial role in the convergence analysis of different types of Q-learning. In this direction the boundedness of the Q-values in the well known Q-learning and double Q-learning is discussed in \cite{doi:10.1137/S0363012997331639,gosavi2006boundedness} and  \cite{xiong2020finite} respectively. 

This manuscript is organized as follows. To make the paper self-contained, basic results as well as the notations used in the paper are presented in Section 2. The proposed algorithms are presented in Section 3. The convergence analysis of the proposed algorithms are discussed in Section 4. The performance of the proposed algorithms tested with a few benchmark problems are  discussed in Section 5. We conclude our discussion in Section 6 by highlighting the conclusion of the proposed work.
\section{Preliminaries}
In this manuscript, an infinite horizon discounted Markov decision process (MDP) with a finite state space $S$, and finite action space, $A$ is considered. The transition probability $p(j|i,a)$ denotes the probability of ending up at state $j$ when action $a$ is picked in state $i$. Let $c(i, a,j) \in \mathbb{R}$ denote the reward for taking an action $a$ in state $i$ and reaching state $j$. Let $\beta \in [0,1)$ be the discount factor. Throughout this paper we assume that the reward is uniformly bounded, i.e., $\exists$  $C_{\max}$ such that  $|c(i,a,j)|\leq C_{\max}$ for all $(i,a,j) \in S \times A \times S $. Further, we assume that the MDP is communicating.  
A policy $\pi$ can be considered as a map from the state to a set of probability distribution over $A$. The objective of the Markov decision problem is to find an optimal policy $\pi^*$.  The existence of such an optimal policy for the finite state action MDP is guaranteed in the classical MDP theory \cite{MR1270015}. The problem of finding an optimal policy reduces to solving the Bellman's optimality equations given by :
\begin{equation}\label{eq1}
Q^*(i,a)=c(i,a)+\beta\, \max_{b \in A}\,Q^*(j,b), \quad \forall \,\,(i,a) \in S \times A,
\end{equation}
where 
\begin{equation*}
c(i,a)=\sum_{j=1}^{|S|}p(j|i,a)c(i,a,j).
\end{equation*}
Corresponding Q-Bellman's operator 
$H:\mathbb{R}^{|S||A|}\rightarrow \mathbb{R}^{|S||A|}$ is defined as \\
\begin{equation*}
(HQ)(i,a)=c(i,a)+\beta \sum_{j=1}^{|S|}p(j|i,a)\max_{b \in A}\,Q(j,b).
\end{equation*}
It is well known that $H$ is  a max-norm contraction with contraction factor $\beta$    \cite{bertsekas1996neuro}. Consequently, it has a unique fixed point, say $Q^*$. 		\\
From this $Q^*$ one can obtain an optimal policy $\pi^*$ as follows

$$ \pi^*(i)\in \text{arg} \max\limits_{a \in A}\, Q^*(i,a), \quad \forall i \in S.$$
\noindent
Also, it is worth mentioning that Q-learning can be considered as a stochastic approximation version of the value iteration \cite{6796861}. \\
Let $Q_n(i,a)$ be the approximation for $Q^*(i,a)$ at step $n$. Using the sample $\{i,a,j,c(i,a,j)\}$ and $Q_n(i,a)$ one can get the improved approximation $Q_{n+1}(i,a)$ for $Q^*(i,a)$ as follows :
\begin{multline*}
Q_{n+1}(i,a) = (1-\alpha_n(i,a))Q_n(i,a)+\alpha_n(i,a)   (c(i,a,j)+\\\beta \, \max_{b \in A} Q_n(j,b)),
\end{multline*}
where $0\leq \alpha_n(i,a) \leq 1$. Using techniques available from stochastic approximation, it is shown that under suitable assumptions the above update rule converges to the fixed point of $H$ with probability one \cite{6796861}. \\
The theoretical convergence of the above iterative procedure is discussed in the literature using various techniques \cite{watkins1992q,6796861,tsitsiklis1994asynchronous,doi:10.1137/S0363012997331639,lee2020unified}. It is worth mentioning that, in \cite{6796861}, by proving the convergence of a generalized version of stochastic iterative procedure, the convergence of Q-learning algorithm is obtained. The convergence result in \cite{6796861} is further relaxed in \cite{singh2000convergence} and finds its application in proving convergence of various RL algorithms \cite{hasselt2010double,de2018multi}. In this manuscript too, the convergence of the proposed iterative procedure is obtained using the result from \cite{singh2000convergence}. In other words, the following lemma from \cite{singh2000convergence} is used as key tool to prove the convergence of the proposed iterative scheme. 

\begin{lemma}(Lemma 1, \cite{singh2000convergence})\label{fl}
	Consider a stochastic process $ \left(\Psi_{n},F_n,\alpha_n\right) $ $n \geq 0$, where $\Psi_{n},F_n,\alpha_n:X\rightarrow \mathbb{R}$ satisfy the following recursive relations :
	\begin{align}
	&\Psi_{n+1}(x)=(1-\alpha_n(x))\Psi_{n}(x)+\alpha_n(x)F_n(x),  \end{align}
	where $x \in X$, $\; n \geq 0.$ Then $\Psi_{n}$ converges to zero with probability one $(w.p.1)$ as $n$ tends to $\infty$, if the following properties hold:\\
	1. The set $X$ is finite.\\
	2. $0\leq\alpha_n(x)\leq 1$, $\sum_{n=1}^{\infty}\alpha_n(x)=\infty$, $\sum_{n=1}^{\infty}\alpha^2_n(x)<\infty$ $w.p.1$.\\
	3. $\|E[F_n|\mathcal{F}_n]\| \leq \kappa \|\Psi_n\|+\zeta_n$, where $\kappa \in[0,1)$ and $\zeta_n$ converges to zero $w.p.1$.\\
	4. $Var[ F_n(x)|\mathcal{F}_n]\leq K(1+ \|\Psi_n\|)^2$, where $K$ is some constant. \\Here $\mathcal{F}_n$ is an increasing sequence of $\sigma$-fields. 
	Further, $\alpha_0$ and $\Psi_0$ are $\mathcal{F}_0$ measurable, and $\alpha_n$, $\Psi_n$, and $F_{n-1}$ are $\mathcal{F}_n$ measurable, for all $n\geq1$. The notation $\|.\|$ refers to any weighted maximum norm. 
\end{lemma}
It can be observed from the literature that boundedness plays an important role in proving the convergence of several stochastic iterative procedure in RL (for example, \cite{azar2011speedy,doi:10.1137/S0363012997331639,singh2000convergence}). It is important to note that boundedness is a necessary condition in the convergence analysis of the proposed algorithms. In this context, the following lemma will be used to show the boundedness of the proposed algorithms
\begin{lemma}(Proposition 3.1 \cite{MR1976398})\label{k1}
	If $\sum_{n=1}^{\infty}|a_n| < \infty$, then the $\prod _{n=1}^{\infty}(1+a_n)$ converges.
\end{lemma}

Considering the applications and advantages of using smooth maximum in place of maximum operator in Bellman's equation, many authors \cite{asadi2017alternative,song2019revisiting,9540362} try to solve the following version of Bellman's equations: 
\begin{align}
&Q(i,a) \nonumber\\=& \,c(i,a)+\beta \sum_{j=1}^{|S|}p(j|i,a)\dfrac{1}{N}\log \sum_{b=1}^{|A|}e^{NQ(j,b)}, \, \forall \,\,(i,a) \in S \times A.
\end{align}
The corresponding smooth Q-Bellman operator $U: R^{|S||A|} \rightarrow R^{|S||A|}$ defined as follows

\begin{equation}
\begin{split}
(UQ)(i,a)=c(i,a)+\beta \sum_{j=1}^{|S|}p(j|i,a)\dfrac{1}{N}\log \sum_{b=1}^{|A|}e^{NQ(j,b)}.
\end{split}
\end{equation}
It is interesting to note that $U$ is a contraction operator with respect to the maximum norm with contraction constant $\beta$ \cite{9540362}. We conclude this section by presenting the relation between the fixed point of the operator $H$ and the fixed point of the operator $U$.
\begin{theorem}(Lemma 3, \cite{9540362})
	Suppose $\hat{ Q }$ and $\tilde{Q}$ are the fixed points of $U$ and $H$ respectively. Then, $$||\textcolor{black}{\hat {Q}-\tilde{Q}}||\leq \dfrac{\textcolor{black}{\beta}}{N(1-\textcolor{black}{\beta})}\log(|A|).$$
\end{theorem}
From the above estimate it is evident that for a large $N$, the fixed point of $U$ is a good approximation for the fixed point of $H$.
\section{Proposed Algorithm}In this section, we present the proposed algorithms. As discussed in the previous section, for Q-learning, at each step of iteration, we generate a reward $c'$ and the next state $s'$. In the proposed algorithm, however, at each step of the iteration, in addition to $c'$ and $s'$, we generate $c''$ and $s''$ by taking an action $a'$ at the state $s'$. Using this information, the update rule for two-step Q-learning is provided in Step 6 of Algorithm 1. 
\begin{algorithm}[h]
	\caption{Two-step Q-learning (TSQL)}
	
	\begin{algorithmic}[1]
		\renewcommand{\algorithmicrequire}{\textbf{Input:}}
		\REQUIRE Initial Q-Vector: $Q_0$; Discount factor: $\beta $, step size rule $\alpha_n$ and sequence $\theta_n$, Number of iterations: $T$, Behaviour policy $\mu$ which generates every state action pairs infinitely often.
		\STATE \textbf{for} $n = 0, 1,\cdots, T-1$ \textbf{do} 
		\STATE \quad Decide $a_n$ at $s_n$ according to policy $\mu$
		\STATE \quad  Observe $s'_n$ and $c'_n$	
		\STATE \quad Decide $a'_n$ at $s'_n$ according to policy $\mu$
		\STATE \quad Observe  $s''_n$ and $c''_n$ 
		\STATE \quad \textbf{Update rule:} \\
		\quad $Q_{n+1}(s_n,a_n)=(1-\alpha_n(s_n,a_n))Q_n(s_n,a_n)+\alpha_n(s_n,a_n)\Big(c'_n+\beta\, \max_{b \in A}Q_n(s'_n,b)$\\
		\hspace{2cm} $+\beta \theta_n (c''_n+\beta\, \max_{e \in A}Q_n(s''_n,e) ) \Big)$
		\STATE \quad $s_{n+1}=s''_n$
		\STATE \textbf{end for}
		\RETURN $Q_{T-1}$
		
	\end{algorithmic}
	
\end{algorithm}\\
As it can be seen from the Algorithm 1, for samples $\{i,a,j,c(i,a,j),d,k,c(j,d,k)\}$ generated at time step $n$, the update rule for TSQL is of the following form:
\begin{equation}
\begin{split}
Q_{n+1}(i,a)=(1-\alpha_n(i,a))Q_n(i,a)+\alpha_n(i,a)(c(i,a,j)+&\\\beta \max_{b \in A}Q_n(j,b)+\beta \theta_n (c(j,d,k)+\beta \max_{e \in A}Q_n(k,e) ) ).
\end{split}
\end{equation}
\begin{remark}\label{rem1}
	Let $\theta_n$ be a sequence of real numbers  satisfying  the following conditions:\\
	(i) $|\theta_n|\leq 1, \quad \forall n \in N.$\\
	(ii) ${|\theta_n|}$ is monotonically decreasing to zero.\\
	(iii) $\sum_{n=1}^{\infty}\alpha_n|\theta_n| < \infty$.\\
	The requirement of these conditions on $\theta_n$ will be made clear in the next section. More specifically, these conditions will be necessary in proving the boundedness of the proposed algorithms.
\end{remark}

\begin{remark}
	The smooth two-step Q-learning algorithm (S-TSQL) is Algorithm 1, with the maximum operator replaced by the LSE function in Step 6. Specifically, for sample $\{i,a,j,c(i,a,j),d,k,c(j,d,k)\}$ generated at time step $n$, the update rule for S-TSQL is of the following form:
	\begin{align}\label{4.5}
&Q_{n+1}(i,a)=(1-\alpha_n(i,a))Q_n(i,a)+\alpha_n(i,a) \nonumber\\
&\hspace*{0.6cm}\left(c(i,a,j)+ \dfrac{\beta}{N}\log \sum_{b=1}^{|A|}e^{NQ_n(j,b)}\right)\\
& \hspace*{1cm}+\beta \theta_n\alpha_n(i,a) \left(c(j,d,k)+ \dfrac{\beta}{N}\log \sum_{e=1}^{|A|}e^{NQ_n(k,e)}\right).\nonumber
	\end{align}
\end{remark}
\begin{remark}
	It is important to mention that, when evaluating the LSE operator in the S-TSQL algorithm, we use an equivalent representation discussed in \cite{MR4328385} to avoid the overflow error that occurs in the direct evaluation of the LSE operator, as mentioned in \cite{MR4328385}. This representation is given by:
	\begin{equation*}
	\dfrac{1}{N}\log \left(\sum_{b=1}^{|A|}e^{NQ(j,b)-Ne}\right)+e,
	\end{equation*}
	where $e = \max_{b\in A} Q(j,b)$.
\end{remark}
\textbf{Assumption 1:} Throughout this paper we assume that $\theta_n$ satisify the conditions mentioned in Remark \ref{rem1}.
\section{Main Results and Convergence Analysis}
In this section, the almost sure convergence of the proposed algorithms are discussed using the stochastic approximation techniques in \cite{singh2000convergence}. 
Before proving the main convergence theorem, first we address the boundedness of Q-values from the two-step Q-learning algorithm. The following lemma ensure the boundedness of the Q-values from the two-step Q-learning algorithm.
\begin{lemma}\label{lemma3}Let $Q_n(i,a)$ be the value corresponding to a state-action pair $(i,a)$ at $n^{th}$ iteration of TSQL. Suppose $\|Q_0\| \leq \dfrac{C_{\max}}{1-\beta}$\,. Then, $\|Q_n\| \leq M, \forall n \in \mathbb{N}, \,\text{where}\,\, M=\dfrac{C_{\max}}{1-\beta}(1+\beta |\theta_0| ) \prod_{i=1}^{\infty} (1+\alpha_i|\theta_i| \beta^2)$.
\end{lemma}
\begin{proof}
	We obtain the proof by applying induction on $n$. Let $(i,a) \in S \times A$
	\begin{dmath*}
		\left|Q_{1}(i,a)\right|=\left|(1-\alpha_0(i,a))Q_0(i,a)+\alpha_0(i,a)\left(c(i,a,j)+\beta \max_{b \in A}Q_0(j,b)\\ +\beta \theta_0 \left(c(j,d,k)+\beta \max_{e \in A}Q_0(k,e) \right) \right) \right|\\
		\leq (1-\alpha_0)\|Q_{0}\| \\+ \alpha_0 \left(C_{\max} + \beta \|Q_{0}\| \right)+ \alpha_0 \beta |\theta_0| (C_{\max}+\beta \|Q_{0}\|)\\
		\leq (1-\alpha_0)\dfrac{C_{\max}}{1-\beta} + \alpha_0 \left(C_{\max} + \beta \dfrac{C_{\max}}{1-\beta} \right)+\\\quad  \alpha_0 \beta |\theta_0|  \left(C_{\max}+\beta \dfrac{C_{\max}}{1-\beta}\right)\\
		\leq (1-\alpha_0)\dfrac{C_{\max}}{1-\beta} + \alpha_0 \left(C_{\max} + \beta \dfrac{C_{\max}}{1-\beta} \right)+ \\ \beta |\theta_0|\left(C_{\max}+\beta \dfrac{C_{\max}}{1-\beta}\right)\\
		\leq\dfrac{C_{\max}}{1-\beta}(1+\beta |\theta_0|).
	\end{dmath*}
	Consequently, $\max_{(i,a)} |Q_1(i,a)| = \|Q_1\| \leq\dfrac{C_{\max}}{1-\beta}(1+\beta |\theta_0|)\leq M$.
	Now assume that  $\|Q_k\| \leq \dfrac{C_{\max}}{1-\beta}(1+\beta |\theta_0| ) \prod_{i=1}^{k-1} (1+\alpha_i|\theta_i| \beta^2)$ is true for $k=1,2,...,n$. Then we will show that $\|Q_{n+1}\| \leq \dfrac{C_{\max}}{1-\beta}(1+\beta |\theta_0| ) \prod_{i=1}^{n} (1+\alpha_i|\theta_i| \beta^2)$.
	Let $L =  \dfrac{C_{\max}}{1-\beta}(1+\beta |\theta_0| ) \prod_{i=1}^{n-1} (1+\alpha_i|\theta_i| \beta^2)$. \\ Consider
	\begin{align*}
		&\bigg|Q_{n+1}(i,a)\bigg|=\bigg|(1-\alpha_n(i,a))Q_n(i,a)+\alpha_n(i,a)\bigg(c(i,a,j)\\&\hspace{0.5cm}+\beta \max_{b \in A}Q_n(j,b)+\beta \theta_n \bigg(c(j,d,k)+\beta \max_{e \in A}Q_n(k,e) \bigg) \bigg) \bigg|\\
		&\leq (1-\alpha_n)\|Q_{n}\| + \alpha_n \bigg(C_{\max} + \beta \|Q_{n}\| \bigg)\\&\hspace*{1cm}+ \alpha_n \beta |\theta_n| (C_{\max}+\beta \|Q_{n}\|)\\
		&\leq (1-\alpha_n)L + \alpha_n \bigg(C_{\max} + \beta L \bigg)+ \alpha_n \beta |\theta_n|  \bigg(C_{\max}+\beta L\bigg)\\
		&= L (1+\alpha_n\beta^2|\theta_n|)-\alpha_n L+\alpha_n C_{\max} + \beta \alpha_n L + \alpha_n \beta |\theta_n|C_{\max}.
	\end{align*}

	Further, we show that, $-\alpha_n L+\alpha_n C_{\max} + \beta \alpha_n L + \alpha_n \beta |\theta_n|C_{\max} \leq 0.$ \\Let $Z_n = -\alpha_n L+\alpha_n C_{\max} + \beta \alpha_n L + \alpha_n \beta |\theta_n|C_{\max}$. Then
	\begin{dmath*}
		Z_n
		= \alpha_n \left( (\beta-1)L+ C_{\max} +  \beta |\theta_n|C_{\max}\right)\\
		=\alpha_n \left( -C_{\max}(1+\beta |\theta_0| ) \prod_{i=1}^{n-1} (1+\alpha_i|\theta_i| \beta^2)+ C_{\max} +  \beta |\theta_n|C_{\max}\right)\\
		\leq \alpha_n \left( -C_{\max}(1+\beta |\theta_0| )+ C_{\max} +  \beta |\theta_n|C_{\max}\right)\\
		\leq 0 \quad(\text{Since},\, |\theta_0|\geq|\theta_n|).
	\end{dmath*}
	Therefore, $\left|Q_{n+1}(i,a)\right| \leq L \,(1+\alpha_n\beta^2|\theta_n|)$. Consequently, $\max_{(i,a)} |Q_{n+1}(i,a)| = \|Q_{n+1}\| \leq \dfrac{C_{\max}}{1-\beta}(1+\beta |\theta_0| ) \prod_{i=1}^{n} (1+\alpha_i|\theta_i| \beta^2)\leq M$. Hence the result follows.
\end{proof}

\begin{theorem}\label{thm2}
	Consider an MDP defined as in Section 2, with a policy $\mu$ capable of visiting every state-action pairs infinitely often. Let $\{i,a,j,c(i,a,j),d,k,c(j,d,k)\}$ be a sample at $n^{th}$ iteration, then the iterative procedure defined as
	\begin{align*}
	&Q_{n+1}(i,a)=(1-\alpha_n(i,a))Q_n(i,a)+\alpha_n(i,a)(c(i,a,j)\\&\hspace{0.5cm}+\beta \max_{b \in A}Q_n(j,b)+\beta \theta_n (c(j,d,k)+\beta \max_{e \in A}Q_n(k,e) ) )
	\end{align*}
	converges to the optimal $Q^*$ $w.p.1$, where $\|Q_0\|\leq \dfrac{C_{\max}}{1-\beta}$, $0\leq \alpha_n(i,a) \leq 1$ $\forall n\in \mathbb{N}$, and $\sum_{n}\alpha_n(i,a)=\infty,\; \sum_{n}\alpha^2_n(i,a)<\infty$. 
\end{theorem}
\begin{proof}
	Let $X=S\times A$. Define $\Psi_n(s,a)=Q_n(s,a)-Q^*(s,a), \forall (s,a) \in S\times A$. For our analyses of this theorem without loss of generality assume that the behaviour policy generate the following trajectory :
	$$s_0,a_0,c'_0,s'_0,a'_0,c_1,s_1,a_1,c'_1,s'_1,a'_1,c_2,s_2,a_2,s'_2,a'_2,c'_2\cdots$$
	Incorporating the above notation and rewriting the update rule we have
	\begin{align*}
	& Q_{n+1}(s_n,a_n)=(1-\alpha_n(s_n,a_n))Q_{n}(s_n,a_n)\\
	& \hspace*{0.6cm}+\alpha_n(s_n,a_n)(c'_n+\beta\, \max_{b \in A}Q_n(s'_{n},b)\\
	&\hspace*{1cm}+\beta \,\theta_n (c_{n+1}+\beta \,\max_{e \in A}Q_n(s_{n+1},e))),
	\end{align*}
	where $c'_n=c(s_n,a_n,s'_{n})$, and $c_{n+1}=c(s'_{n},a'_{n},s_{n+1})$ .
	Hence,
	\begin{align*}
	&	\Psi_{n+1}(s_n,a_n)=(1-\alpha_n(s_n,a_n))\Psi_{n}(s_n,a_n)\\
	& \hspace{0.6cm}+\alpha_n(s_n,a_n)(c'_n+\beta\, \max_{b \in A}Q_n(s'_{n},b)\\
	& \hspace{1cm}+\beta \,\theta_n (c_{n+1}+\beta \,\max_{e \in A}Q_n(s_{n+1},e) )-Q^*(s_n,a_n)).
	\end{align*}
	Let
	\begin{align*}
	& F_n(s_n,a_n) =  c(s_n,a_n,s'_{n})+\beta\, \max_{b \in A}Q_n(s'_{n},b)\\
	&+\beta \,\theta_n (c(s'_{n},a'_{n},s_{n+1})+\beta \,\max_{e \in A}Q_n(s_{n+1},e) )-Q^*(s_n,a_n). 
	\end{align*}
	Since the updating of a single iterate of two-step Q-learning includes simulation of  extra sample compared to Q-learning, define an increasing sequence of sigma fields $\mathcal{F}_n$ as follows, for $n=0$, let $\mathcal{F}_0= \sigma(\{Q_0,s_0,a_0,\alpha_0\})$ and for $n\geq 1$,  $\mathcal{F}_n=\sigma(\{Q_0,s_0,a_0,\alpha_0, c'_{j-1},s'_{j-1},a'_{j-1}, c_j,\alpha_j,s_j,a_j : 1 \leq j \leq n \})$.  
	With this choice of $\mathcal{F}_n$, $\alpha_0$ and $\Psi_0$ will be $\mathcal{F}_0$ measurable, and $\alpha_n$, $\Psi_n$, and $F_{n-1}$ are $\mathcal{F}_n$ measurable. Now we have
	\begin{align*}
	&\left|E[F_n(s_n,a_n)|\mathcal{F}_n]\right|\\
	&=\bigg{|}E\Bigg[ c(s_n,a_n,s'_{n})+\beta\, \max_{b \in A}Q_n(s'_{n},b)+\beta \,\theta_n \Bigg(c(s'_{n},a'_{n},s_{n+1})\\&\hspace{1cm}+\beta \,\max_{e \in A}Q_n(s_{n+1},e) \Bigg)-Q^*(s_n,a_n)|\mathcal{F}_n\Bigg]\bigg{|}\\
	&=\left|E\left[c(s_n,a_n,s'_n)+\beta\, \max_{b \in A}Q_n(s'_n,b)-Q^*(s_n,a_n)|\mathcal{F}_n\right]\right.\\&\hspace*{0.6cm}\left.+\beta \theta_n E\left[c(s'_{n},a'_{n},s_{n+1})+\beta \,\max_{e \in A}Q_n(s_{n+1},e)|\mathcal{F}_n\right]\right|\\
	&\leq\left|E\left[c(s_n,a_n,s'_n)+\beta\, \max_{b \in A}Q_n(s'_n,b)-Q^*(s_n,a_n)|\mathcal{F}_n\right]\right|\\&\hspace*{0.6cm}+\beta |\theta_n| \left|E\left[\left(c(s'_{n},a'_{n},s_{n+1})+\beta \,\max_{e \in A}Q_n(s_{n+1},e) \right)|\mathcal{F}_n\right]\right|\\
	&=\left|\sum_{j=1}^{|S|}p(j|s_n,a_n)\left(c(s_n,a_n,j)+\beta \max_{b \in A}Q_n(j,b)-Q^*(s_n,a_n)\right)\right|\\&\hspace*{0.6cm}+\beta |\theta_n| \left|E\left[\left(c(s'_{n},a'_{n},s_{n+1})+\beta \,\max_{e \in A}Q_n(s_{n+1},e) \right)|\mathcal{F}_n\right]\right|\\
	&\leq \left| HQ_n(s_n,a_n)-HQ^*(s_n,a_n)\right|+ \beta\, |\theta_n| \left(C_{\max} + \beta \|Q_n\|\right)\\
	& \leq \beta \, \|\Psi_{n}\| + \zeta_n,
	\end{align*}
	\noindent
	where $\zeta_n=\beta\, |\theta_n| (C_{\max} + \beta \|Q_n\|)$. Since $\|Q_n\|$ is bounded by $M$, and $|\theta_n|$ converges to zero, $\zeta_n$ converges to zero as $n \rightarrow \infty$.
	Therefore, condition $(3)$ of Lemma \ref{fl} holds. \\
	Now consider
	\begin{align*}	
	&Var[F_n(s_n,a_n)|\mathcal{F}_n]\\
	&=E\left[\left(F_n(s_n,a_n)-E[F_n(s_n,a_n)|\mathcal{F}_n]\right)^2 |\mathcal{F}_n\right]\\
	&\leq E[(c(s_n,a_n,s'_{n})+\beta \max_{b \in A}Q_n(s'_{n},b)\\&\hspace*{1cm}+\beta \theta_n (c(s'_{n},a'_{n},s_{n+1})+\beta \max_{e \in A}Q_n(s_{n+1},e) ))^2|\mathcal{F}_n]\\
	& \leq \left(C_{\max}+ \beta \|Q_n\|+ \beta |\theta_n| \left(C_{\max}+ \beta \|Q_n\|\right)\right)^2\\
	& \leq  3\left(C_{\max}^2+ \beta^2 \|Q_n\|^2+ \beta^2 (C_{\max}+ \beta \|Q_n\|)^2\right)\\
	& \leq  3\left(C_{\max}^2+ \beta^2 \|Q_n\|^2+ 2\beta^2  (C^2_{\max}+ \beta^2 \|Q_n\|^2)\right)\\
	& =  3C_{\max}^2+ 6\beta^2  C^2_{\max}+ (3 \beta^2 + 6  \beta^4) \|Q_n\|^2\\
	& \leq  3C_{\max}^2+ 6\beta^2  C^2_{\max} + 2 (3 \beta^2 + 6  \beta^4) (\|\Psi_n\|^2 + \|Q^*\|^2)\\
	&= 3C_{\max}^2+ 6\beta^2  C^2_{\max} + 2 (3 \beta^2 + 6  \beta^4) \|Q^*\|^2 + 2 (3 \beta^2 + 6  \beta^4) \|\Psi_n\|^2\\
	& \leq K(1 + \|\Psi_n\|^2 ) \leq K(1 + \|\Psi_n\| )^2,
	\end{align*}
	where $K = \max\{3C_{\max}^2+ 6\beta^2  C^2_{\max}+2 (3 \beta^2 + 6  \beta^4) \|Q^*\|^2  ,\, 2 (3 \beta^2 + 6  \beta^4)\}$. For the choice $X=S\times A$, $\kappa = \beta$ all the hypotheses of Lemma \ref{fl} is satisfied. Consequently, $\Psi_n$ converges to zero $w.p.1$. Hence, $Q_n$ converges to $Q^*$ $w.p.1.$
	
\end{proof}
Now we will outline the convergence analysis of smooth two-step Q-learning (S-TSQL). First, the following lemma provides the boundedness of Q-iterates obtained via S-TSQL.
\begin{lemma}Let $Q_n(i,a)$ be the value corresponding to a state-action pair $(i,a)$ at $n^{th}$ iteration of S-TSQL. Suppose $\|Q_0\| \leq \dfrac{C_{\max}}{1-\beta}$. Then $\|Q_n\| \leq D,\, \forall n \in \mathbb{N},\, \text{where} \;D =\left(\dfrac{C_{\max}}{1-\beta}+\dfrac{\log|A|}{N(1-\beta)}\right)(1+\beta |\theta_0| ) \prod_{i=1}^{\infty} (1+\alpha_i|\theta_i| \beta^2)$.
\end{lemma}
\begin{proof}
	We obtain the proof by applying induction on $n$. Let $(i,a) \in S \times A$
		\begin{align*}
		&|Q_{1}(i,a)|=\bigg|(1-\alpha_0(i,a))Q_0(i,a)+\alpha_0(i,a)\bigg(c(i,a,j)\\&\hspace{0.6cm}+\dfrac{\beta}{N}\log \sum_{b=1}^{|A|}e^{NQ_0(j,b)}+\beta \theta_0 \bigg(c(j,d,k)\\&\hspace*{1.4cm}+\dfrac{\beta}{N}\log \sum_{e=1}^{|A|}e^{NQ_0(k,e)} \bigg) \bigg)\bigg|
	\end{align*}

	\begin{align*}\label{4.6}
	&\leq (1-\alpha_0)\|Q_{0}\| + \alpha_0 \left(C_{\max} + \beta \left|\dfrac{1}{N}\log \sum_{b=1}^{|A|}e^{NQ_0(j,b)}\right| \right)\\
	&\hspace*{0.6cm}+ \alpha_0 \beta |\theta_0| \left(C_{\max}+\beta \left|\dfrac{1}{N}\log \sum_{e=1}^{|A|}e^{NQ_0(k,e)}\right|\right).
	\end{align*}
	Note that, 
	\begin{align*}
	\left|\dfrac{1}{N}\log \sum_{b=1}^{|A|}e^{NQ_0(j,b)}\right|\leq \dfrac{1}{N}\log \sum_{b=1}^{|A|}e^{N\|Q_0\|}= \|Q_0\|+\dfrac{\log|A|}{N}.
	\end{align*}
	Hence,
	\begin{align*}
	&|Q_{1}(i,a)|\leq (1-\alpha_0)\|Q_0\|+ \alpha_0 \left(C_{\max} + \beta\left(\|Q_0\|+\dfrac{\log|A|}{N}\right)\right)\\&\hspace{1.5cm}+ \alpha_0 \beta |\theta_0|  \left(C_{\max}+\beta \left(\|Q_0\|+\dfrac{\log|A|}{N}\right)\right)\\
	&\leq (1-\alpha_0)\dfrac{C_{\max}}{1-\beta}+(1-\alpha_0)\dfrac{\log|A|}{N} \\&\hspace{0.5cm}+ \alpha_0 \left(C_{\max} + \beta\left(\dfrac{C_{\max}}{1-\beta}+\dfrac{\log|A|}{N}\right)\right)\\&\hspace{1.5cm}+ \alpha_0 \beta |\theta_0|  \left(C_{\max}+\beta \left(\dfrac{C_{\max}}{1-\beta}+\dfrac{\log|A|}{N}\right)\right)\\
	&=(1-\alpha_0)\dfrac{C_{\max}}{1-\beta} + \alpha_0 C_{\max} + \alpha_0\beta \dfrac{C_{\max}}{1-\beta} \hspace{1cm}\\&+\alpha_0 \beta |\theta_0| \left(C_{\max}+\beta \dfrac{C_{\max}}{1-\beta}\right)\\&\hspace{0.4cm}+(1-\alpha_0)\dfrac{\log|A|}{N}+\alpha_0\beta \dfrac{\log|A|}{N}+\alpha_0\beta^2|\theta_0|\dfrac{\log|A|}{N}\\
	&\leq \left(\dfrac{C_{\max}}{1-\beta}+\dfrac{\log|A|}{N}\right)(1+\beta |\theta_0| )\\&\leq\left(\dfrac{C_{\max}}{1-\beta}+\dfrac{\log|A|}{N(1-\beta)}\right)(1+\beta |\theta_0| ) .
	\end{align*}
	Consequently, $\max_{(i,a)} |Q_1(i,a)| = \|Q_1\| \leq \left(\dfrac{C_{\max}}{1-\beta}+\dfrac{\log|A|}{N(1-\beta)}\right)(1+\beta |\theta_0| ) \leq D$.
	Assume that  $\|Q_k\| \leq \left(\dfrac{C_{\max}}{1-\beta}+\dfrac{\log|A|}{N(1-\beta)}\right)(1+\beta |\theta_0| ) \prod_{i=1}^{k-1} (1+\alpha_i|\theta_i| \beta^2)$  is true for $k=1,2,...,n$. Then we will show that $\|Q_{n+1}\| \leq \left(\dfrac{C_{\max}}{1-\beta}+\dfrac{\log|A|}{N(1-\beta
		)}\right)(1+\beta |\theta_0| ) \prod_{i=1}^{n} (1+\alpha_i|\theta_i| \beta^2)$.
	Let $L =   \left(\dfrac{C_{\max}}{1-\beta}+\dfrac{\log|A|}{N(1-\beta
		)}\right)(1+\beta |\theta_0| ) \prod_{i=1}^{n-1} (1+\alpha_i|\theta_i| \beta^2)$. Now,\\
	\begin{dmath*}
		\left|Q_{n+1}(i,a)\right|\leq (1-\alpha_n)\|Q_{n}\| + \alpha_n \left(C_{\max} + \beta \left(\|Q_{n}\|+\dfrac{\log|A|}{N}\right) \right)\\+ \alpha_n \beta |\theta_n| \left(C_{\max}+\beta \left(\|Q_{n}\|+\dfrac{\log|A|}{N}\right)\right)\\
		\leq (1-\alpha_n)L + \alpha_n \left(C_{\max} + \beta L +\dfrac{\beta\log|A|}{N} \\+ \beta |\theta_n|  C_{\max}+\beta^2 |\theta_n| L+ \beta^2 |\theta_n|\dfrac{\log|A|}{N}\right)\\
		= L \,(1+\alpha_n\beta^2|\theta_n|)-\alpha_n L\\+\alpha_n \left(C_{\max} + \beta L +\dfrac{\beta\log|A|}{N} + \beta |\theta_n|  C_{\max}+ \beta^2 |\theta_n|\dfrac{\log|A|}{N}\right).
	\end{dmath*}
	Using similar arguments as that of the proof of Lemma \ref{lemma3}, one can show that $Z_n \leq 0,$ where  $Z_n =-\alpha_n L+\alpha_n \left(C_{\max} + \beta L +\dfrac{\beta\log|A|}{N} + \beta |\theta_n|  C_{\max}+ \beta^2 |\theta_n|\dfrac{\log|A|}{N}\right)$. 
	Hence, $\left|Q_{n+1}(i,a)\right| \leq  L \,(1+\alpha_n\beta^2|\theta_n|)$. Consequently, $\|Q_{n+1}\| \leq \left(\dfrac{C_{\max}}{1-\beta}+\dfrac{\log|A|}{N(1-\beta
		)}\right)(1+\beta |\theta_0| ) \prod_{i=1}^{n} (1+\alpha_i|\theta_i| \beta^2)\leq D$. Hence the result follows.
\end{proof}

\begin{theorem}\label{thm3}
	Consider an MDP defined as in Section 2, with a policy $\mu$ capable of visiting every state-action pairs infinitely often. Let $\{i,a,j,c(i,a,j),d,k,c(j,d,k)\}$ be a sample at the $n^{th}$ iteration, then the iterative procedure defined as
	\begin{align*}
	&Q_{n+1}(i,a)=(1-\alpha_n(i,a))Q_n(i,a)+\alpha_n(i,a)\\
	&\hspace*{0.6cm}\left(c(i,a,j)+ \dfrac{\beta}{N}\log \sum_{b=1}^{|A|}e^{NQ_n(j,b)}\right)\\
	&\hspace*{1cm}+\beta \theta_n\alpha_n(i,a) \left(c(j,d,k)+ \dfrac{\beta}{N}\log \sum_{e=1}^{|A|}e^{NQ_n(k,e)}\right)
	\end{align*}
	converges to the fixed point of $U$, where $\|Q_0\|\leq\dfrac{C_{\max}}{1-\beta}$, $0\leq \alpha_n(i,a) \leq 1$ $\forall n\in \mathbb{N}$ , and $\sum_{n}\alpha_n(i,a)=\infty,\; \sum_{n}\alpha^2_n(i,a)<\infty$. 
\end{theorem}
\begin{proof}
	The proof runs on similar lines of the Theorem \ref{thm2} and hence is omitted.
\end{proof}
	\section{Experiments}
\noindent
The performance of the proposed algorithms TSQL and S-TSQL compared with the classical Q-learning (QL) and double Q-learning (D-Q) \cite{hasselt2010double}, double Q-learning with average estimator and double step-size (D-Q-Avg) \cite{wentaoweng}, SOR Q-learning (SORQL) \cite{kamanchi2019successive} algorithms by solving the three benchmark problems in this section. In all the experiments for S-TSQL, the choice of $N$ is $10000$. First the proposed algorithms are tested with an MDP discussed in \cite{wentaoweng} is known for the maximisation bias. Next the performance of the proposed algorithms are tested by solving $100$ random MDPs generated by the MDP toolbox \cite{toolbox}. Finally, the performance of the TSQL and S-TSQL are examined with the classic roulette problem discussed as a multi-armed bandit problem in \cite{8695133}. Our implementation of the experiments can be found here \footnote{\url{https://github.com/shreyassr123/Two-step-Q-learning}.}
\subsection{Maximization Bias Example} Recently, W. Weng et al. \cite{weng} discussed a maximization bias example similar to Example 6.1 in \cite{MR3889951} to demonstrate the modification of a double Q-learning algorithm. More specifically, it was observed in \cite{weng} that by doubling the step size and averaging the two estimators in double Q-learning, one can improve the performance of the classical double Q-learning. Throughout this paper, we use D-Q-Avg to denote the algorithm suggested by W. Weng et al. \cite{weng}.
In this example, there are nine states labelled as $\{0,1,2,...,8\}$, and the agent always starts in state zero. Each state has two actions: going left or going right. If the agent takes the right action in state zero, it receives zero reward, and the game ends. If it takes the left action in state zero, it transitions to any of the remaining states $\{1,2,...,8\}$ with equal probability and receives a zero reward. If the agent chooses the action right from these new states $\{1,2,...,8\}$, it returns to state zero. If it chooses the left action from these new states $\{1,2,...,8\}$, the game ends. Both actions from these new states $\{1,2,...,8\}$ lead to a reward drawn from a normal distribution with a mean of $-0.1$ and a variance of $1$. The code for demonstrating the proposed algorithms built on the code available in \cite{wentaoweng}. Throughout this example the policy is set to be epsilon-greedy with an epsilon value $0.1$. During the implementation of all algorithms, the Q-values are initialised to zero and trained over $200$ episodes. We plot the graph for the number of episodes versus the probability of choosing the left action from the zero state at the end of each episode. The probability is taken as an average of over $1000$ independent runs of the experiment. Note that, opting to go right consistently maximizes the mean reward for the agent. Therefore, a higher probability of choosing left suggests that the algorithm has learned a suboptimal policy. 

\begin{figure}[h!]
	\includegraphics[scale=0.4]{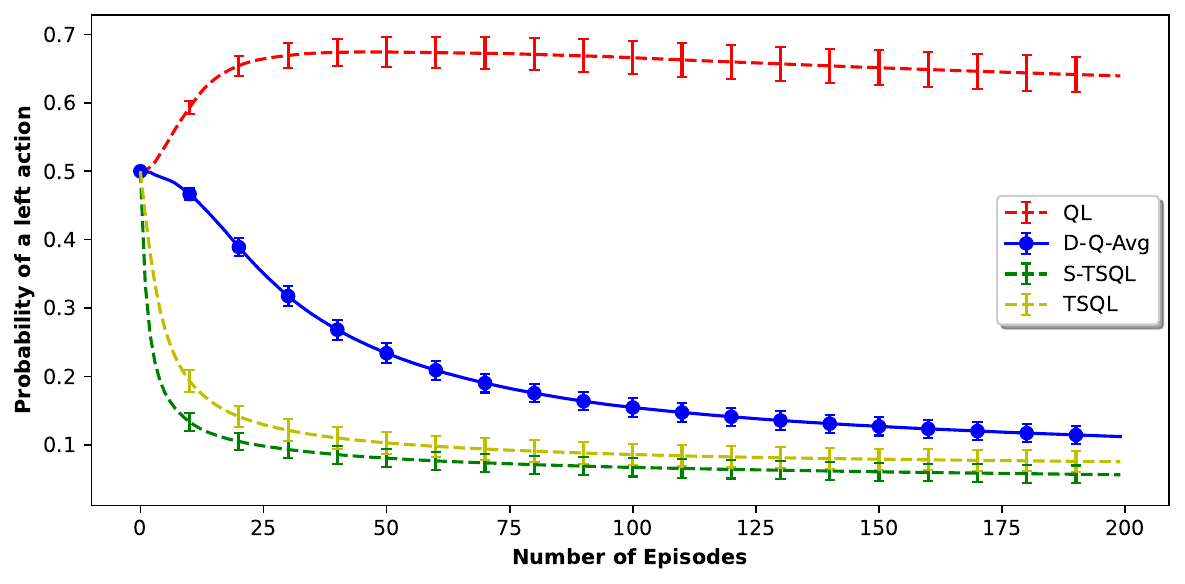}
	\centering
	\caption{Performance of QL, D-Q-Avg, S-TSQL, and TSQL with $\alpha_n=\frac{1}{n+1}$, and $\theta_n=\frac{1}{n^2+10}$.}
	\label{key1}
\end{figure}
\begin{figure}[h!]
	\includegraphics[scale=0.4]{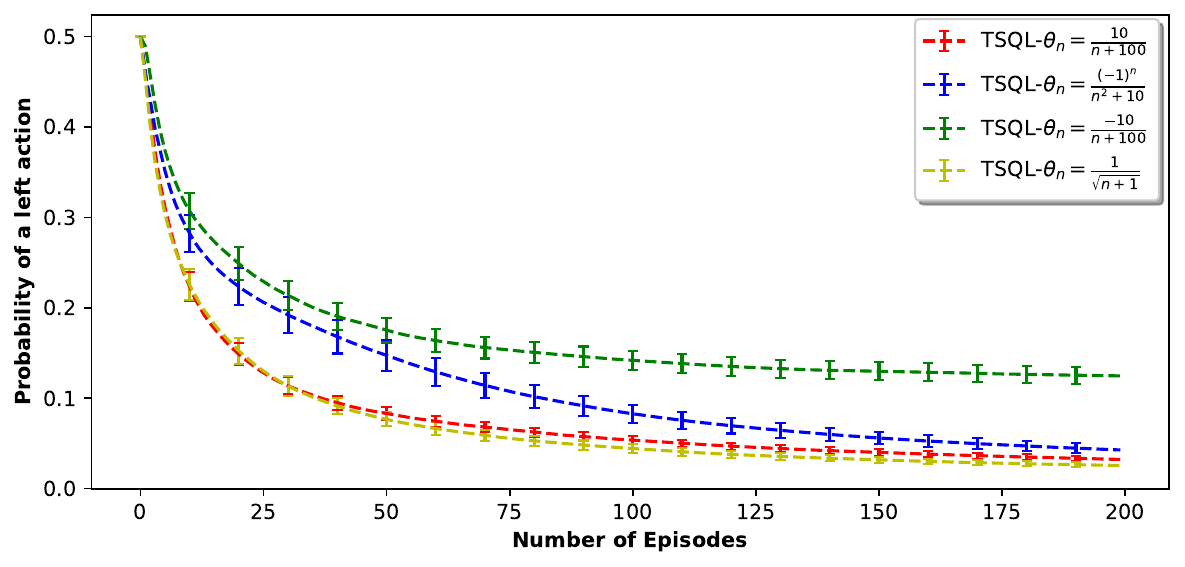}
	\centering
	\caption{Performance of TSQL with $\alpha_n=\frac{10}{n+100}$, and various choice of $\theta_n$.}
	\label{key2}
\end{figure}
\begin{figure}[h!]
	\includegraphics[scale=0.4]{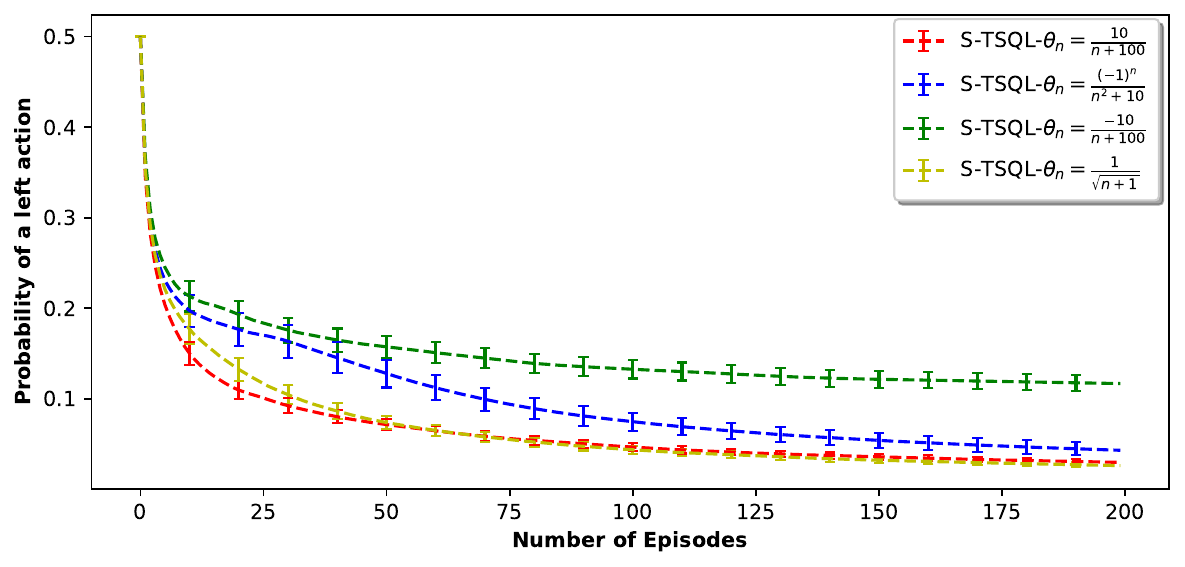}
	\centering
	\caption{Performance of S-TSQL with $\alpha_n=\frac{10}{n+100}$, and various choice of $\theta_n$.}
	\label{key3}
\end{figure}
\begin{figure}[h!]
	\includegraphics[scale=0.4]{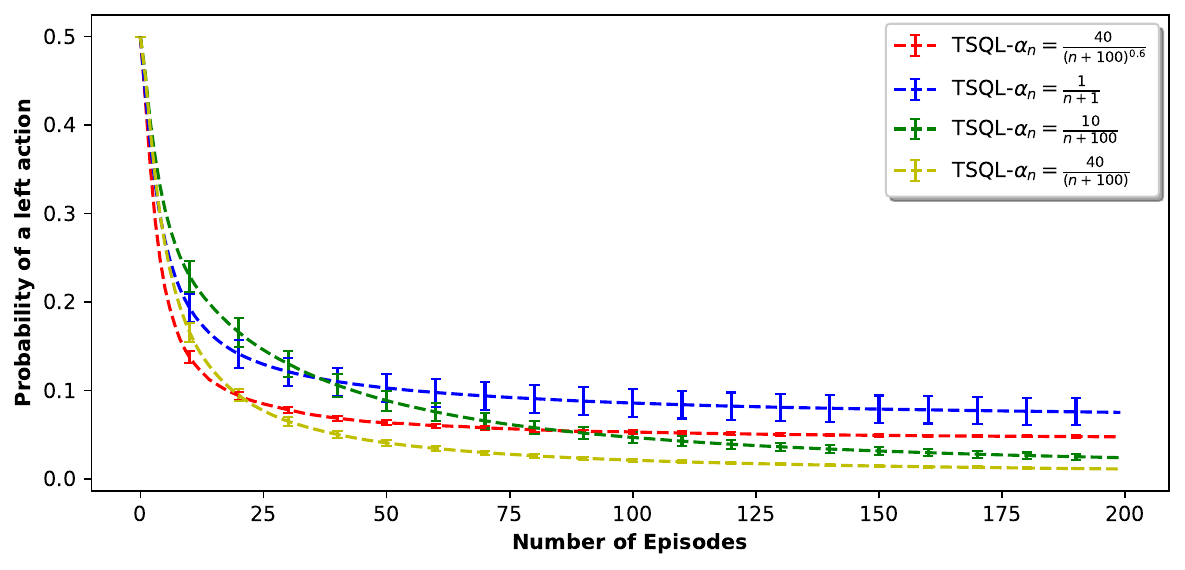}
	\centering
	\caption{Performance of TSQL with $\theta_n=\frac{1}{n^2+10}$, and various choice of $\alpha_n$.}
	\label{key4}
\end{figure}
\begin{figure}[h!]
	\includegraphics[scale=0.4]{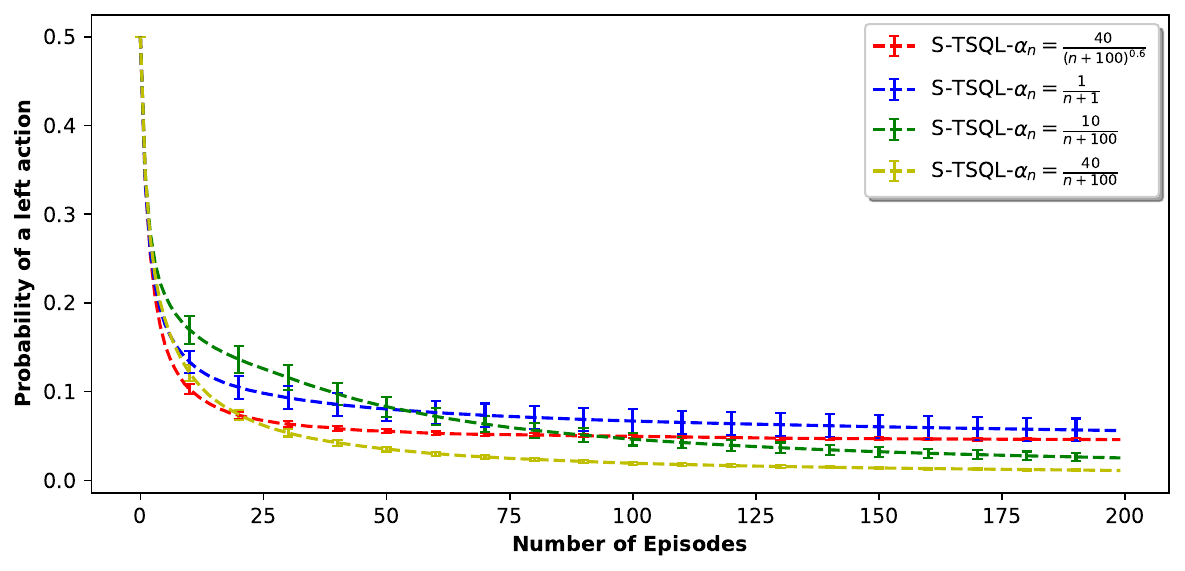}
	\centering
	\caption{Performance of S-TSQL with $\theta_n=\frac{1}{n^2+10}$, and various choice of $\alpha_n$.}
	\label{key5}
\end{figure}

%
%
%

%
%
\noindent
Figure \ref{key1}, provides the performance comparison of the algorithms QL, D-Q-Avg \cite{weng}, and the proposed algorithm TSQL and S-TSQL for the learning rate $\alpha_n=\frac{1}{n+1}$. Moreover, the parameter $\theta_n$ in the proposed algorithms TSQL and S-TSQL takes the value $\frac{1}{n^2+10}$. From Figure \ref{key1}, one can see that the performance of the classical QL algorithm deteriorate due to the maximization bias. The improved version of double Q-learning algorithm D-Q-Avg \cite{weng}, performs better than the QL algorithm. It is interesting to note that the proposed algorithms TSQL and S-TSQL performs better than the QL and D-Q-Avg \cite{weng} algorithms. In other words, the proposed algorithms TSQL and S-TSQL requires less number of episodes to learn the optimal policy. Among the proposed algorithms TSQL and S-TSQL, S-TSQL performs better than the TSQL. Figure \ref{key2}, and Figure \ref{key3} provides the performance of TSQL and S-TSQL respectively for the fixed learning rate $\alpha_n = \frac{10}{n+100}$ and various choice of $\theta_n$. Similarly, Figure \ref{key4}, and Figure \ref{key5} provides the performance of TSQL and S-TSQL respectively for the choice of $\theta_n=\frac{1}{n^2+10}$ and various choices for the learning rate.
\subsection{Random MDPs }
In this experiment, $100$ random MDPs with $10$ states and $5$ actions are generated using the MDP toolbox \cite{toolbox}. The discount factor $\beta$ is set to $0.6$. Comparison of the proposed algorithms with QL and D-Q is provided. 
The average error for these 100 MDPs are calculated as follows
\begin{equation}\label{error}
\text{Average Error}=\dfrac{1}{100}\sum_{k=1}^{100}\|J_{k}^*-\max_{b \in A}\,Q_{k}^*(.,b)\|,
\end{equation}
where $Q_k^*(.,.)$ is the Q-value at the end of $10000$ iterations obtained using the particular algorithm, and $J_k^*$ represents the optimal value function derived using standard value iteration from the toolbox. We initialize the Q-values of all the algorithms to zero. The choice of $\alpha_n$ and $\theta_n$ for this experiment is $\frac{1}{(n+2)^{0.501}}$, and $\frac{1000}{n+1000}$ respectively. In addition, epsilon-greedy policy is followed to update the Q-values with $\epsilon=0.1$.
\begin{table}[ht]
	\centering
	
	\begin{tabular}{@{}lccccccccl@{}}\toprule	
		
		\textbf{Algorithm} & \textbf{Average Error}   \\\midrule
		\textbf{QL}  & 0.499\\
		\textbf{TSQL} &  0.2331\\
		\textbf{S-TSQL} &  0.2331\\
		\textbf{D-Q}   & 1.3914   \\\bottomrule
		
	\end{tabular}	
	
	\vspace{0.1cm}
	
	\caption{Comparison of algorithms for $S=10$, $A=5$, and $\beta=0.6$ averaged over $100$ MDPs.}
	\label{tab:my-table}
	
\end{table}
\noindent
Table \ref{tab:my-table} shows that the average error of the proposed algorithms is much less than that of the Q-learning and double Q-learning. It is interesting to note that the proposed algorithms show more than fifty percentage decrease in the average error. In this example the proposed TSQL and S-TSQL algorithms show almost the same performance. One can also observe that the D-Q does not perform well in this example.\\
Lately, C. Kamanchi et al. \cite{kamanchi2019successive} proposed a modification in the QL called as SOR Q-learning (SORQL) to solve a special class of MDP. More specifically, MDPs whose transition probability $p$ satisfies $p(i|i,a)>0, \forall (i,a)\in S\times A$ can be solved using SORQL algorithm. In this experiment, once again $100$ random MDPs with this special class are generated with $10$ states and $5$ actions using the toolbox \cite{toolbox}. The choice of discount factor for all these experiments is $0.6$. At the end of $100000$ iterations, the average error for all the algorithms are calculated using Equation \ref{error}. The learning rate for all the algorithms is $\frac{1}{(n+2)^{0.501}}$. The parameter $\theta_n$ in TSQL and S-TSQL takes the value $ \frac{1000}{n+1000}$.
\begin{table}[]
	\centering
	
	\begin{tabular}{@{}lccccccccl@{}}\toprule	
		
		\textbf{Algorithm} & \textbf{Average Error}   \\\midrule
		\textbf{QL}  & 0.4033\\
		\textbf{TSQL} &  0.2945\\
		\textbf{S-TSQL} &  0.2945\\
		\textbf{D-Q}   & 0.6713   \\
		\textbf{SORQL}   & 0.3065\\\bottomrule
	\end{tabular}	
	
	\vspace{0.1cm}
	
	\caption{Comparison of algorithms for $S=10$, $A=5$ averaged over $100$ MDPs.}
	\label{tab:my-table1}
\end{table}\\
Table \ref{tab:my-table1}, provides the performance of QL, D-Q \cite{hasselt2010double}, SOR Q-learning (SORQL) \cite{kamanchi2019successive} and the proposed TSQL and S-TSQL algorithms. From, Table \ref{tab:my-table1}, one can conclude that the proposed algorithms perform better than the algorithms QL, D-Q and SORQL. It is worth mentioning that throughout this experiment, SORQL algorithm uses the optimal relaxation parameter. In general, finding the optimal relaxation factor is not practically possible but using the idea discussed in \cite{gametheory} one can get a better relaxation parameter.
\subsection{Roulette as Multi-Armed Bandit problem}
In this example, there is a single state with multiple actions, a total of $39$ possible actions that an agent can choose. This example is an elementary version of the roulette game, and the best strategy in gambling is not to gamble. Among the $39$ actions in the MDP, there is an action corresponding to not gambling, which ends the episode with zero rewards. Every other action results in a reward from a normal distribution with a mean of $-0.0526$ and a variance of one. The discount factor is set to $0.99$. This example is discussed in \cite{8695133} and it is also similar to that of the MDP discussed as a part of the loop bandit MDP in (\cite{hasselt}, Fig 4.3). By solving this MDP, the performance of the algorithms QL, D-Q and proposed algorithms TSQL and S-TSQL are compared. All the algorithms follow an exploratory policy. We run all the algorithms for one-lakh episodes and take the average of ten such independent experiments.
\begin{figure}[h]
	\includegraphics[scale=0.4]{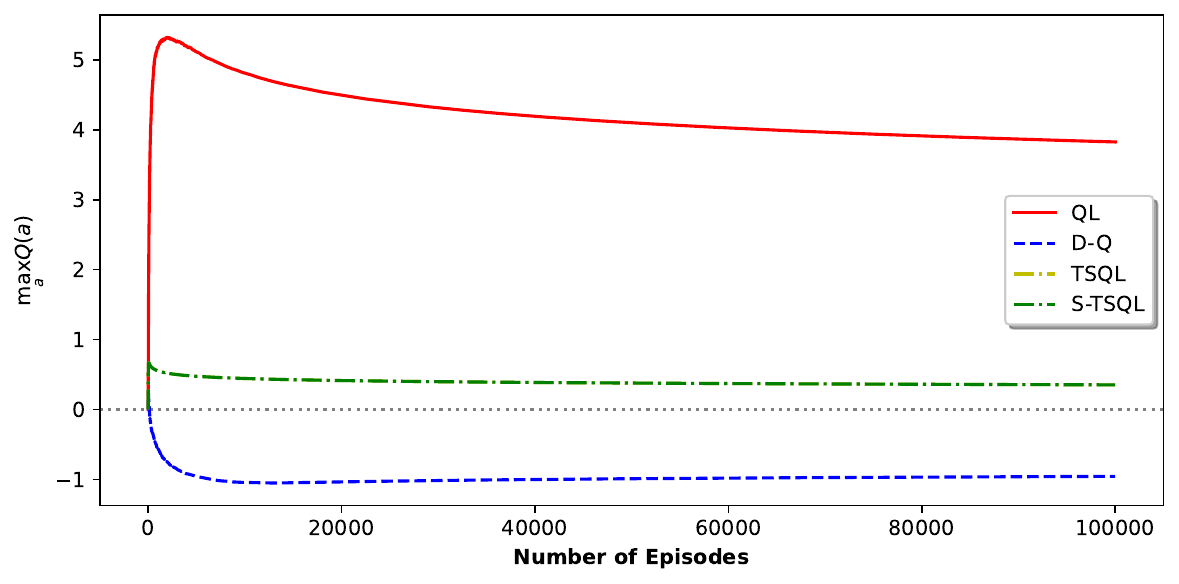} 
	\centering
	\caption{Performance of QL, D-Q, TSQL, S-TSQL with $\alpha_n=\frac{10}{n+100}$, and $\theta_n=\frac{-10^3}{n+10^3}$. }	 	\label{key8}
\end{figure}
\begin{figure}[h]
	\includegraphics[scale=0.4]{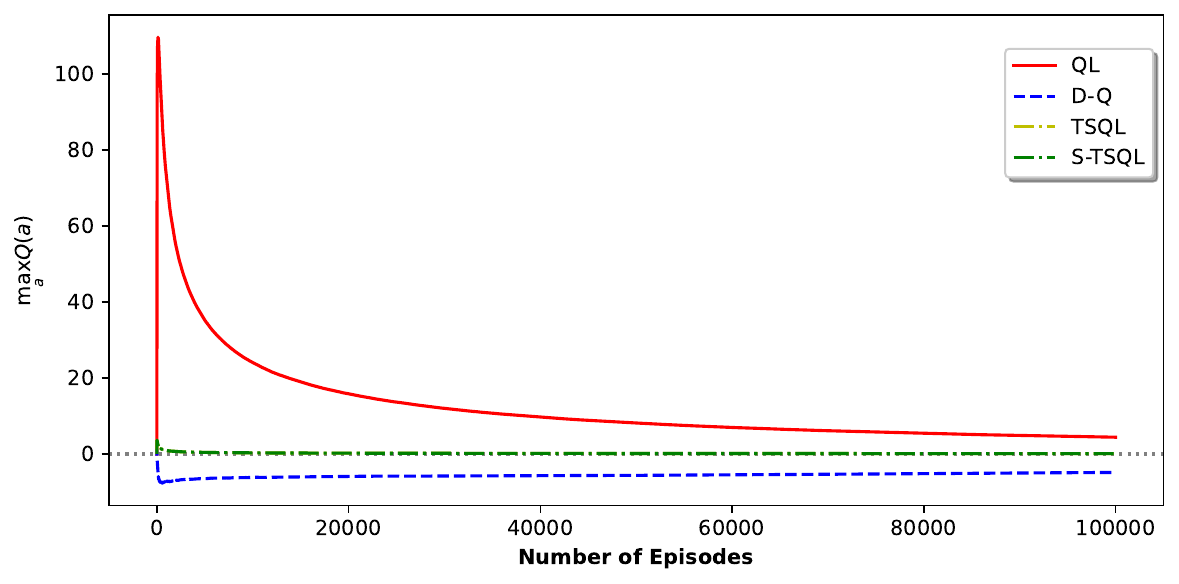} 
	\centering
	\caption{Performance of QL, D-Q, TSQL, S-TSQL with $\alpha_n=\frac{100}{n+100}$, and $\theta_n=\frac{-10^3}{\sqrt{n}+10^3}$. }
	\label{key9}
\end{figure}
Figure \ref{key8} shows the behaviour of all the algorithms for $100000$ episodes with learning rate $\alpha_n=\frac{10}{n+10}$ and the parameter $\theta_n$ in TSQL and S-TSQL, takes the value $\frac{-10^3}{n+10^3}$. In Figure \ref{key8}, the learning rate $\alpha_n$ for all the algorithms  is $\frac{10}{n+100}$ and the sequence ${\theta_n}$ for the proposed  algorithms is $\frac{-10^3}{n+10^3}$. As expected, QL overestimates the optimal Q-value. At the end of one hundred thousand episodes, the Q-value corresponding to QL is still far from the optimal value. The D-Q algorithm underestimates the optimal Q-value. At the end of one hundred thousand episodes and ten independent experiments the average value of $\max_a Q(a)$ for QL, D-Q and proposed algorithms is $3.82$, $-0.95$, and $0.352$ respectively. One can improve the performance of the proposed algorithm TSQL and S-TSQL by changing the learning rate $\alpha_n$ and the parameter $\theta_n$ in the algorithms. More specifically, Figure \ref{key9}, shows the behaviour of the algorithms QL, D-Q, TSQL, and S-TSQL after $100000$ episodes for the learning rate $\alpha_n =\frac{100}{n+100}$ and $\theta_n=\frac{-10^3}{\sqrt{n}+10^3}$. In this setting at the end of one hundred thousand episodes and ten independent experiments, the average value of $\max_a Q(a)$ is $4.41$ and $-4.85$ respectively, for the method QL and D-Q which is far away from the optimal value $0$. However, the proposed algorithms produce the value for $\max_a Q(a)$ as $0.088$ after one hundred thousand episodes and ten independent experiment. This value is very close to the optimal value zero.
\section{Conclusion}
\noindent
In this paper, novel off-policy two-step algorithms are proposed and the proposed algorithms are robust and easy to implement. The proposed two-step Q-learning algorithms are robust and easy to implement. Results regarding the boundedness of the two-step Q-learning and smooth two-step Q-learning algorithm are provided, along with the results of convergence. To demonstrate the superior performance of the proposed algorithms, several empirical tests are conducted. Specifically, the experiments are performed on the classic bias examples discussed in \cite{hasselt,8695133,weng}, and two sets of one hundred MDPs are generated randomly. Theoretical understanding of the choice of sequence ${\theta_n}$ on bias in general is an interesting question for further research.

\end{document}